\documentclass[preprint,12pt,3p]{elsarticle}
\usepackage{amsmath,amssymb,amsthm}
\usepackage{booktabs}
\usepackage{graphicx}
\usepackage{bm}
\usepackage[hidelinks]{hyperref}

\newtheorem{theorem}{Theorem}
\newtheorem{lemma}{Lemma}
\newtheorem{proposition}{Proposition}
\newtheorem{definition}{Definition}
\theoremstyle{remark}
\newtheorem*{remark}{Remark}

\journal{Chaos, Solitons \& Fractals}

\begin{document}

\begin{frontmatter}

\title{Identifiability of a dissipative knowledge-dynamics model:\\
exact recovery under designed excitation,\\
degeneration on observational data}

\author[inn]{Arman A. Kostanian\corref{cor}}
\ead{a.kostanian@innopolis.university}
\author[hse]{Armen L. Beklaryan}

\address[inn]{Innopolis University, Innopolis, Russia}
\address[hse]{HSE University, Moscow, Russia}
\cortext[cor]{Corresponding author.}

\begin{abstract}
Human learning is a dissipative dynamical process: mastery accumulates through practice, decays
through forgetting, and propagates across interdependent concepts. We model it as a nonlinear
dissipative system of ordinary differential equations whose parameters are mechanistically
meaningful --- a transfer-of-practice matrix, per-concept forgetting
rates, and a saturating practice-response gain --- and we study when those parameters can actually
be recovered from data.

For two concepts, we prove a constructive structural identifiability theorem for the associated inverse problem under explicit excitation conditions, together
with monotonicity, robustness and L-stability results. We derive an implicit L-stable scheme
for the dissipative subsystem and a batched solver numerically equivalent to the per-trajectory
formulation (bit-exact predictions, gradients to $10^{-10}$) yet two orders of magnitude faster,
making estimation feasible on cohorts of $2.5\times10^{4}$ learners.

The empirical study is two-sided. Under the theorem's excitation conditions, synthetic recovery is
exact: parameters to machine precision, prerequisite structure at $F_1 = 1.0$. On large
observational benchmarks it is not. An apparently strong recovery --- forgetting rates correlating
with topic difficulty at Spearman $\rho = 0.83$ --- is refuted by four independent controls: it
survives destroying the temporal order of the data, is matched by a classical Bayesian baseline,
and is unaffected by removing real timestamps. A single shared forgetting rate attains the same accuracy, and replaying the real schedule with identical true rates reproduces most of the correlation. We trace this to the stationary structure of the
model and show that it is consistent with the degeneration anticipated by the theory in the absence of designed excitation. The result
delineates a sharp boundary between identifiable and unidentifiable regimes and yields a
validation protocol for interpretability claims.
\end{abstract}

\begin{keyword}
nonlinear dynamics \sep dissipative systems \sep structural identifiability \sep
inverse problems \sep learning dynamics \sep numerical integration
\end{keyword}

\end{frontmatter}

%% =====================================================================
\section{Introduction}
%% =====================================================================

Learning is a complex dynamical process. A learner's competence is distributed over many
interdependent concepts and evolves in continuous time under competing mechanisms: reinforcement
through practice, dissipation through forgetting, and transfer along prerequisite relations, all
observed only indirectly, through a sparse stream of binary responses. Systems of this kind ---
high-dimensional, dissipative, partially observed, driven by an irregular exogenous input --- are
the natural object of nonlinear dynamics, and they arise throughout the study of complex social
phenomena.

The engineering counterpart of this problem is known as \emph{knowledge tracing}: inferring a
learner's latent state from their response history in order to adapt instruction
\cite{corbett1994knowledge,abdelrahman2023survey}. The field is now dominated by recurrent and
attentional neural models \cite{piech2015deep,zhang2017dkvmn,ghosh2020akt}, which predict responses
accurately but represent the learner state as an uninterpretable latent vector. Three consequences
follow. The state cannot be read: one cannot say what a learner knows or why a prediction was made.
Forgetting, when modelled at all, is inserted as a heuristic decay term chosen to improve fit
rather than derived and tested \cite{nagatani2019forgetting}. And personalisation requires
centralised retraining on raw learner logs.

A natural response is to replace the black box with a mechanistic model: a differential system
whose parameters are, by construction, the quantities of interest --- a transfer matrix encoding
prerequisite coupling, per-concept forgetting rates, a practice-response gain. This is the model we
develop in Section~\ref{sec:model}. But building such a model raises a question that is rarely
asked in this setting, and which we take as our subject: naming a parameter ``forgetting rate''
does not make it one. Whether the estimated value of that parameter has any relation to forgetting
depends on whether the parameter is \emph{identifiable} --- uniquely determined by the
observations --- and identifiability of a dynamical system is never unconditional. It depends on
how the system was excited while it was observed.

This dependence is well understood in systems biology, viral dynamics and control theory
\cite{bellman1970structural,ljung1994global,walter1997identification}, where structural
identifiability is analysed as a property of a model \emph{together with} its input, and where the
analysis is routinely used to design experiments that make parameters recoverable
\cite{villaverde2018exciting,villaverde2019inputdependent}. It has not, to our knowledge, been
brought to bear on models of knowledge dynamics --- even as that field increasingly reports
``interpretable'' parameters and reads scientific meaning into their values.

We therefore study a dissipative ODE model of learning from the standpoint of the inverse problem
it defines, and the picture that emerges has two sides.

On one side, the inverse problem is solvable, and we characterise exactly when. We prove a
structural identifiability theorem: under a practice schedule containing rest intervals and
intervals of isolated, non-constant practice of each concept, the parameters are uniquely
determined, with a constructive closed-form reconstruction in the two-concept case, accompanied by
monotonicity, robustness and L-stability results. Numerically, recovery under these conditions is
exact: parameters to machine precision, and prerequisite coupling with $F_1 = 1.0$ on synthetic
cohorts.

On the other side, the conditions matter. Standard educational datasets are \emph{observational}:
they record whatever practice learners happened to undertake, and they do not contain the regimes
the theorem requires. Trained on such data, the model produces a parameter vector that correlates
strongly with an external reference --- recovered forgetting rates track measured topic difficulty
at Spearman $\rho = 0.83$ --- which, by the evidential standards common in this literature, would
be reported as successful recovery of an interpretable quantity. It is not. Four independent
controls show that the correlation survives destroying the temporal order of the data, is
reproduced by a three-decade-old Bayesian baseline, and is unaffected by supplying real elapsed
time. We trace the effect to the stationary structure of the equations, which forces the decay
parameter to absorb per-topic response frequency, and we show that it is precisely the
degeneration that the identifiability analysis anticipates in the absence of excitation; the link is an analogy rather than a formal implication, and a formal treatment under discrete binary observation is given in \cite{kostanian2026exact}.

The two sides together delineate a sharp qualitative boundary in the same dynamical system: a
regime in which mechanism is recoverable from observation, and a regime in which it is not, with
the transition governed not by the volume of data but by the design of its collection. We regard
this boundary, rather than any single fitted model, as the contribution of this paper --- together
with its practical corollary, that claims of parameter interpretability should be accompanied by
controls capable of detecting the failure we document.

\paragraph{Contributions}
\begin{enumerate}
\item A dissipative-ODE model of knowledge dynamics with a structurally factorised,
prerequisite-aware transfer matrix and mechanistically meaningful parameters
(Section~\ref{sec:model}).
\item A structural identifiability theorem for the associated inverse problem under explicit
excitation conditions, constructive for $K=2$, with monotonicity, robustness and L-stability
results (Section~\ref{sec:ident}).
\item Numerical methods for the forward and inverse problems: an L-stable implicit scheme for
the dissipative subsystem, a memory-efficient adjoint formulation, and a batched solver proved
numerically equivalent to the per-trajectory formulation and two orders of magnitude faster, which
is what makes the reproducibility and control experiments feasible (Section~\ref{sec:numerics}).
\item A two-sided empirical study on synthetic and four real benchmarks: exact recovery under
excitation, and --- established by four independent controls --- degeneration of recovery into a
reparameterisation of response frequencies on observational data, together with a structural
explanation of why this occurs (Sections~\ref{sec:setup}--\ref{sec:results}).
\item A minimal validation protocol for claims of parameter interpretability, applied to our own
model, which reverses a conclusion that conventional evidence would have supported
(Sections~\ref{sec:setup} and \ref{sec:discussion}).
\end{enumerate}

The remainder of the paper is organised as follows. Section~\ref{sec:related} positions the work,
Section~\ref{sec:model} introduces the model, Section~\ref{sec:ident} the identifiability theory,
and Section~\ref{sec:numerics} the numerical methods. Sections~\ref{sec:setup} and
\ref{sec:results} describe the experimental protocol and results,
Section~\ref{sec:discussion} discusses implications and limitations, and
Section~\ref{sec:conclusion} concludes.

%% =====================================================================
\section{Related work}\label{sec:related}
%% =====================================================================

\subsection{Knowledge tracing: from interpretable to accurate}

Modelling the evolution of a learner's mastery from a sequence of graded responses has two
traditions. Bayesian Knowledge Tracing (BKT) \cite{corbett1994knowledge} represents each skill by a
two-state hidden Markov model with parameters --- prior, learn, slip, guess --- that have direct
pedagogical readings, and remains in production use in intelligent tutoring systems. Deep Knowledge
Tracing \cite{piech2015deep} and its successors \cite{zhang2017dkvmn,pandey2019sakt,ghosh2020akt}
replace the interpretable state with a recurrent or attentional latent vector, gaining substantial
predictive accuracy at the cost of any mechanistic reading of the state.

A large body of work adds forgetting to both traditions, motivated by the classical retention curve
\cite{ebbinghaus1913memory}: decay terms conditioned on elapsed time or repetition counts
\cite{nagatani2019forgetting,settles2016spaced}, memory-augmented and graph-based variants, and
concept-graph models that alternate forgetting, aggregation and update phases. These models
introduce forgetting as an architectural component chosen to improve fit; the resulting decay
parameters are typically inspected qualitatively rather than tested for recoverability. Our work
differs in exactly this respect: we ask under which conditions such a decay parameter is
\emph{identifiable}, and what its estimate means when those conditions fail.

\subsection{Identifiability in student modelling}\label{sec:related:ident}

Identifiability is not a new concern in this field, and we make no claim of priority for raising
it. Beck and Chang \cite{beck2007identifiability} observed that BKT models with different
parameters can produce identical predictions, and framed this as an identifiability problem; van de
Sande \cite{vandesande2013properties} characterised the resulting families of observationally
equivalent parameters analytically. Doroudi and Brunskill \cite{doroudi2017misidentified}
subsequently showed that the phenomenon so described is not, strictly, non-identifiability:
invoking results on the identifiability of hidden Markov models, they proved that BKT is
identifiable under mild parameter conditions. They instead named the real difficulty
\emph{semantic model degeneracy} --- the parameters that best fit the data can be inconsistent
with the conceptual assumptions of the model --- and demonstrated it by fitting BKT to data
generated from alternative learning processes.

Our contribution sits on this line and extends it in two directions. First, the existing analysis
concerns a \emph{discrete-time hidden Markov} model with per-skill parameters; we treat a
\emph{continuous-time dissipative ODE} with coupled concepts, where identifiability is a question
about an inverse problem for a differential system rather than about an HMM likelihood, and where
the answer depends explicitly on the practice schedule. To our knowledge this is the first
structural identifiability result of that kind for knowledge dynamics. Second, we show that
semantic degeneracy is not merely a possibility to be guarded against but the \emph{generic
outcome} on standard observational benchmarks: a parameter named ``forgetting rate'' is driven, at
the optimum, to encode per-topic answer frequency. We further show that this degeneracy is
invisible to the validation practices currently used in the field (Section~\ref{sec:related:val}).

\subsection{Structural identifiability and excitation in dynamical systems}

Outside education, structural identifiability is a mature subject in systems biology, viral
dynamics and control theory. The standard notion is that a parameter is structurally identifiable
if its value can in principle be determined from noise-free observation of the model output
\cite{bellman1970structural,ljung1994global,walter1997identification}; if it is not, its estimated
numerical value carries no meaning and model-based conclusions about unobserved states may be
wrong \cite{villaverde2018exciting}. Crucially, structural identifiability is conditional on the
\emph{input}: results are stated under sufficiently exciting --- persistently exciting --- inputs,
and methods exist to determine how exciting an input must be for a given model, precisely so that
experiments can be designed to make parameters identifiable
\cite{villaverde2018exciting,villaverde2019inputdependent}. The distinction between structural and
practical identifiability \cite{raue2009structural,wieland2021identifiability} is central to what
follows: our theorem establishes the former, and our experiments show that the latter can
nonetheless fail on passively collected data. The same conditionality appears in the theory of
ill-posed inverse problems, where regularisation compensates for information the data do not
contain \cite{tikhonov1977illposed}.

Closest to the present work, Norden \emph{et al.} \cite{norden2025identifiability} argue that
structural identifiability matters for machine learning with partially observed dynamical systems,
using identifiability analysis to relate parameter configurations that produce identical outputs.
Our aim is complementary: rather than exploiting known unidentifiability to constrain a predictor,
we derive the identifiability conditions for a specific dissipative model and then use them
diagnostically, to ask whether a widely used class of observational datasets satisfies them --- and
to show what the estimated parameters mean when it does not.

We import this apparatus into learning dynamics. The practice schedule plays the role of the
input; Theorem~\ref{thm:ident} states the excitation regimes under which the parameters of the
knowledge-dynamics system are uniquely recoverable, and gives the reconstruction in closed form for
the two-concept case. The empirical part of the paper is then, in the language of that literature,
an assessment of whether standard educational logs constitute a sufficiently exciting input. They
do not.

\subsection{Validating ``interpretability'': current practice}\label{sec:related:val}

Interpretable knowledge tracing typically supports its claims by correlating a learned quantity
with an external reference: skill difficulty, expert prerequisite annotations, or review intervals.
Our position is not that such results are wrong, but that the evidential standard is too weak to
distinguish recovery from artefact. We show (Section~\ref{sec:results:controls}) that a strong
correlation of this type --- $\rho = 0.83$ between a recovered forgetting rate and topic
difficulty --- survives the destruction of all temporal ordering in the data and is matched by a
classical BKT baseline, and therefore establishes nothing about the dynamics the model claims to
capture. Accordingly we propose, and apply to our own model, a minimal protocol of a shuffle
control and a simple-baseline control (Section~\ref{sec:setup:protocol}).

\subsection{Physics-informed and continuous-time models}

Physics-informed neural networks embed governing equations into learning, typically as residual
penalties in the loss \cite{raissi2019pinn}. Neural ODEs \cite{chen2018neuralode} provide the
continuous-time counterpart of deep architectures. Our model is physics-informed in the
\emph{hard-constraint} sense: rather than penalising deviation from an assumed equation, we
restrict the hypothesis class to a fixed dissipative ODE whose parameters are the objects of
interest. This distinction matters here, since all results reported below are obtained with the
residual weight set to zero; the mechanistic content resides in the model structure, not in a soft
penalty. We are not aware of prior work that applies structural identifiability analysis to a
continuous-time model of knowledge dynamics.

%% =====================================================================
\section{The Cognitive-PINN model}\label{sec:model}
%% =====================================================================

\subsection{State and governing equation}

Let a curriculum consist of $K$ concepts.\footnote{The model name is historical and alludes to
the physics-informed learning paradigm. We stress that the model is physics-informed in the
\emph{hard-constraint} sense: the governing equation defines the hypothesis class rather than
entering the loss as a residual penalty as in \cite{raissi2019pinn}. In all experiments reported
here the residual weight is zero, and the mechanistic content resides entirely in the model
structure.} We describe a learner by a mastery state
$\bm{k}(t) \in [0,1]^K$, where $k_j(t)$ is competence on concept $j$, evolving as

\begin{equation}
\frac{d\bm{k}}{dt}
= \underbrace{\bm{A}\,\bm{\phi}\!\left(\bm{p}(t)\right) \odot \left(\bm{1}-\bm{k}\right)}_{\text{acquisition}}
\;-\;\underbrace{\bm{\Lambda}\odot\bm{k}}_{\text{dissipation}}
\;+\;\underbrace{\bm{\sigma}\odot d\bm{W}_t}_{\text{fluctuation}},
\label{eq:main}
\end{equation}

where $\odot$ is the Hadamard product, $\bm{p}(t)\in\mathbb{R}_{\ge 0}^K$ is the practice intensity
--- the exogenous input --- $\bm{A}\in\mathbb{R}_{\ge 0}^{K\times K}$ is a transfer matrix,
$\bm{\Lambda}\in\mathbb{R}_{>0}^{K}$ collects decay rates, and $\bm{W}_t$ is a Wiener process. In
the analysis of Sections~\ref{sec:ident}--\ref{sec:numerics} we set $\bm{\sigma}=\bm{0}$; the
stochastic term models trial-to-trial fluctuation and is not estimated in this work.

\paragraph{Acquisition} Practice raises mastery with diminishing returns: the factor
$(\bm{1}-\bm{k})$ enforces saturation at $k_j=1$ and keeps the state in the unit box. The response
is itself saturating,
\begin{equation}
\phi_j(p) = 1 - e^{-\alpha_j p}, \qquad \alpha_j > 0,
\label{eq:phi}
\end{equation}
increasing and concave with $\phi_j(0)=0$: additional repetitions within a session contribute
progressively less.

\paragraph{Dissipation} Unpractised competence decays at rate $\Lambda_j$. This makes the system
dissipative in the strict sense --- without input, $\bm{k}$ contracts exponentially toward the
origin.

\paragraph{Coupling} Off-diagonal entries $A_{ij}$ transfer the effect of practising concept $j$ to
concept $i$; the diagonal carries direct acquisition.

\subsection{Structure of the transfer matrix}

Left unconstrained, $\bm{A}$ has $K^2$ free entries and is poorly behaved. We factorise
\begin{equation}
\bm{A} = \bm{D}^{1/2}\left(\bm{Q}_{\mathrm{prior}}\odot\bm{B}\right)\bm{D}^{1/2},
\label{eq:A}
\end{equation}

The factor $\bm{B}$ is parameterised as a non-negative sum of a low-rank term and a diagonal,
\begin{equation}
\bm{B} = \operatorname{softplus}\!\left(\bm{U}\bm{V}^{\top}\right)
+ \operatorname{diag}\!\left(\bm{d}\right),
\qquad \bm{U},\bm{V}\in\mathbb{R}^{K\times r},\quad \bm{d}>0 ,
\label{eq:B}
\end{equation}
where $r \ll K$ is the rank of the decomposition and $\bm{d}$ carries direct acquisition. The
element-wise $\operatorname{softplus}$ guarantees $B_{ij}\ge 0$ and hence $A_{ij}\ge 0$ --- the
condition required by Lemma~\ref{lem:box} for invariance of the unit box. Non-negativity of the
transfer coefficients is therefore not postulated but enforced by the parameterisation; the same
property reflects the substantive requirement that practising one concept cannot lower mastery of
another.
where $\bm{Q}_{\mathrm{prior}}\in\{0,1\}^{K\times K}$ is an optional mask of known prerequisite
relations, and the symmetric scaling
with diagonal $\bm{D}$ normalises the influence a concept can exert and receive. Under the
input-mediated coupling of \eqref{eq:main} the state equations are decoupled for a given input, so
this scaling plays no role in the stability of the integration schemes of Section~\ref{sec:numerics}. In all experiments below we set
$\bm{Q}_{\mathrm{prior}}=\bm{1}$: supplying an expert prerequisite graph and subsequently
``recovering'' it would make structural validation circular.

\subsection{Observation model}

The state is latent; what is observed is a sequence of graded responses. For item $m$ at time
$t_m$, with $\bm{q}_m\in\{0,1\}^K$ the required concepts,
\begin{equation}
\Pr\left(y_m = 1 \mid \bm{k}(t_m)\right)
= \sigma\!\left(\beta_m\left(\bar{k}_m - \tfrac{1}{2}\right) - \delta_m\right),
\qquad
\bar{k}_m = \frac{\bm{q}_m^{\top}\bm{k}(t_m)}{\max\left(\bm{q}_m^{\top}\bm{1},1\right)},
\label{eq:emission}
\end{equation}
with $\sigma$ the logistic function and $(\beta_m,\delta_m)$ item discrimination and difficulty in
the manner of item-response theory; $\beta_m>0$. In all experiments of Section~\ref{sec:results} these parameters are indexed by topic, one pair $(\beta_j,\delta_j)$ per topic, so that $\delta_j$ absorbs the mean difficulty of the items of topic $j$. For single-concept items \eqref{eq:emission}
reduces to the standard score on the practised concept.

\begin{remark}[identifiability of the observation model]
In item-response models with free item parameters, the scale and location of the latent variable
are determined only up to an affine transformation: a shift and a rescaling of $\bm{k}$ can be
absorbed by $(\beta_m,\delta_m)$. Here that indeterminacy is removed not by a normalisation
convention but by the dynamics themselves. The factor $(\bm{1}-\bm{k})$ in the acquisition term and
the term $-\bm{\Lambda}\odot\bm{k}$ pin both boundaries of the state space: by Lemma~\ref{lem:box}
the trajectory remains in $(0,1)^K$, and $k_j=0$, $k_j=1$ are substantively defined states --- no
mastery and full mastery. The scale and location of $\bm{k}$ are therefore fixed by the structure
of~\eqref{eq:main}, and $(\beta_m,\delta_m)$ describe item properties at a fixed state scale. The
constraint $\beta_m>0$ additionally excludes a reflection of the scale.

We stress that Theorem~\ref{thm:ident} is a \emph{structural} result, stated in the idealisation of
continuous observation of $\bm{k}(\cdot)$. Practical identifiability under sparse binary
observations through~\eqref{eq:emission} is a separate question, which this paper addresses
empirically: its failure on observational data is documented in
Sections~\ref{sec:results:apparent}--\ref{sec:results:why}.
\end{remark}

\subsection{Estimation problem}

Estimation minimises the negative log-likelihood of \eqref{eq:emission} under trajectories
generated by \eqref{eq:main},
\begin{equation}
\mathcal{L}(\theta)
= -\frac{1}{M}\sum_{s}\sum_{m}\log \Pr\!\left(y^{(s)}_m \mid \bm{k}^{(s)}(t_m;\theta)\right)
+ \mathcal{R}(\theta),
\qquad \theta = \left(\bm{B}, \bm{\Lambda}, \bm{\alpha}\right).
\label{eq:loss}
\end{equation}

Two remarks matter for the results below. First, the model is physics-informed in the
hard-constraint sense: the governing equation is not a residual penalty added to the loss but is
built into the hypothesis class, so every candidate trajectory solves \eqref{eq:main} by
construction. All experiments therefore use a zero residual weight, and this makes the objective
identical for the per-trajectory and batched formulations of Section~\ref{sec:numerics}. Second,
$\mathcal{R}$ is zero unless stated otherwise: an $\ell_1$ penalty on $\bm{B}$ is natural if sparse
structure is expected, but it drives the off-diagonal of $\bm{A}$ toward zero and thereby
suppresses precisely the structure a recovery experiment is meant to test.

\subsection{Interpretation and its limits}\label{sec:model:limits}

Every parameter has a stated mechanistic reading: $\bm{\Lambda}$ as forgetting rates, $\bm{A}$ as
prerequisite transfer, $\bm{\alpha}$ as practice responsiveness. This motivates the model, and it
is also what this paper subjects to scrutiny. Under sustained practice at intensity $\bm{p}$,
\eqref{eq:main} admits the element-wise steady state
\begin{equation}
k^{*}_j = \frac{\left[\bm{A}\bm{\phi}(\bm{p})\right]_j}
{\left[\bm{A}\bm{\phi}(\bm{p})\right]_j + \Lambda_j},
\label{eq:steady}
\end{equation}
so $\Lambda_j$ enters the predicted response probability of concept $j$ directly. Fitting
\eqref{eq:loss} therefore exerts pressure on $\Lambda_j$ to reproduce the observed success rate on
that concept, irrespective of whether the data contain information about decay over time. Whether
the estimate is a forgetting rate or a reparameterisation of difficulty is not settled by the
model's construction; it is an identifiability question.

%% =====================================================================
\section{Structural identifiability}\label{sec:ident}
%% =====================================================================

\subsection{Setting and definition}

For $K=2$ the deterministic part of \eqref{eq:main} reads
\begin{equation}
\begin{aligned}
\dot k_1 &= \left[A_{11}\phi_1(p_1(t)) + A_{12}\phi_2(p_2(t))\right](1-k_1) - \Lambda_1 k_1, \\
\dot k_2 &= \left[A_{21}\phi_1(p_1(t)) + A_{22}\phi_2(p_2(t))\right](1-k_2) - \Lambda_2 k_2,
\end{aligned}
\label{eq:K2}
\end{equation}
with unknown $\theta = (A_{11},A_{12},A_{21},A_{22},\alpha_1,\alpha_2,\Lambda_1,\Lambda_2)
\in \Theta := \mathbb{R}_{+}^{8}$. We treat $\bm{k}(\cdot)$ and $\bm{p}(\cdot)$ as observed, the
standard idealisation of structural identifiability; observation through the discrete emission
\eqref{eq:emission} is addressed in Lemma~\ref{lem:robust}.

\begin{definition}[Indistinguishability]
Two parameterisations $\theta,\theta'\in\Theta$ are \emph{indistinguishable} if for every
admissible schedule $\bm{p}(\cdot)$ and every $\bm{k}_0\in(0,1)^2$ the trajectories coincide for
all $t\ge 0$. The model is \emph{structurally identifiable} if indistinguishability implies
$\theta=\theta'$.
\end{definition}

\subsection{Excitation conditions}

Identifiability is a property of the model together with its input. We require the schedule to
contain:
\begin{itemize}
\item \textbf{(R) Rest.} An interval $I_R$ of positive length with $p_1 \equiv p_2 \equiv 0$.
\item \textbf{(S1) Isolated non-constant practice of concept 1.} An interval $I_1$ of positive
length with $p_1 > 0$, $p_2 \equiv 0$, and $p_1$ non-constant,
$\sup_{I_1}p_1 > \inf_{I_1}p_1$.
\item \textbf{(S2)} The symmetric condition for concept 2.
\end{itemize}

Each regime removes part of the dynamics: (R) switches off acquisition, exposing pure dissipation;
(S1) switches off the second input channel, exposing the response of both concepts to a single
driving signal. The requirement that $p_1$ be non-constant separates the gain $\alpha_1$ from the
transfer weight $A_{11}$: a single practice intensity cannot distinguish a weakly responding
concept driven hard from a strongly responding concept driven gently.

\subsection{Two lemmas}

\begin{lemma}[Invariance of the unit box]\label{lem:box}
If $\bm{k}(0)\in(0,1)^2$, $A_{ij}\ge 0$, $\Lambda_i>0$ and $p_j\ge 0$, then
$\bm{k}(t)\in(0,1)^2$ for all $t\ge 0$.
\end{lemma}

\begin{proof}
On every interval of constant input the equation for $k_i$ is linear, $\dot k_i=u_i-s_ik_i$ with $u_i=\sum_jA_{ij}\phi_j(p_j)\ge0$ and $s_i=u_i+\Lambda_i>0$, with solution $k_i(t)=k_i^\ast+(k_i(t_0)-k_i^\ast)e^{-s_i(t-t_0)}$, $k_i^\ast=u_i/s_i\in[0,1)$. Thus $k_i(t)$ is a convex combination of $k_i(t_0)\in(0,1)$ and $k_i^\ast\in[0,1)$ with positive weight on $k_i(t_0)$, so $k_i(t)\in(0,1)$; the claim follows by induction over intervals. Under zero input the face $k_i=0$ is invariant rather than repelling, but a trajectory from the interior does not reach it, since $k_i(t)=k_i(t_0)e^{-\Lambda_i(t-t_0)}>0$.
\end{proof}

Lemma~\ref{lem:box} guarantees that the divisions by $k_i$ and $1-k_i$ below are well defined. The
next lemma is the analytical core of the argument.

\begin{lemma}[Monotone signature of the saturating response]\label{lem:H}
For $r>1$ let
\begin{equation}
\tilde H_r(u) := \frac{1-e^{-ru}}{1-e^{-u}}, \qquad u>0 .
\label{eq:H}
\end{equation}
Then $\tilde H_r$ is strictly decreasing on $(0,\infty)$, with
$\lim_{u\to 0^+}\tilde H_r(u)=r$ and $\lim_{u\to\infty}\tilde H_r(u)=1$; hence it is a bijection
from $(0,\infty)$ onto $(1,r)$.
\end{lemma}

\begin{proof}
Writing $\tilde H_r' = N(u)/(1-e^{-u})^2$ and multiplying the numerator by $e^u>0$ gives
$f(u) := r e^{-(r-1)u} - (r-1)e^{-ru} - 1$, with $f(0)=0$ and
$f'(u) = r(r-1)\left[e^{-ru}-e^{-(r-1)u}\right]<0$ for $u>0$, $r>1$. Hence $f(u)<0$, so $N(u)<0$
and $\tilde H_r'(u)<0$. The limits follow by L'H\^opital's rule.
\end{proof}

$\tilde H_r(u)$ is the ratio of the response produced by practice of intensity $ru$ to that
produced by intensity $u$. In the weak-practice limit it equals $r$ --- the response is locally
linear --- while under saturation it tends to $1$. Strict monotonicity means this ratio is a
\emph{signature} of the saturating nonlinearity: observing how the response scales between two
intensities determines the operating point $\alpha_1 p$ uniquely. This is exactly the information a
constant schedule cannot provide.

\subsection{The identifiability theorem}

\begin{theorem}[Structural identifiability, $K=2$, constructive]\label{thm:ident}
Under conditions (R), (S1) and (S2), the model \eqref{eq:K2} is structurally identifiable, and
every component of $\theta$ admits a closed-form expression in the observed trajectory and
schedule.
\end{theorem}

\begin{proof}
Let $\theta,\theta'$ be indistinguishable; we exhibit $\theta$ explicitly.

\emph{Step 1 (dissipation, from R).} On $I_R$ both responses vanish and \eqref{eq:K2} reduces to
$\dot k_i = -\Lambda_i k_i$, so $k_i(t)=k_i(t_R^-)e^{-\Lambda_i(t-t_R^-)}$ with $k_i(t_R^-)>0$ by
Lemma~\ref{lem:box}. Hence for any $t\in(t_R^-,t_R^+]$,
\begin{equation}
\Lambda_i = -\frac{1}{t-t_R^{-}}\log\frac{k_i(t)}{k_i(t_R^{-})}, \qquad i=1,2 .
\label{eq:lam}
\end{equation}

\emph{Step 2 (self-transfer and gain, from S1).} On $I_1$, $\phi_2=0$. With $\Lambda_1$ known,
define the observable
\begin{equation}
g_1(t) := \frac{\dot k_1(t) + \Lambda_1 k_1(t)}{1-k_1(t)} = A_{11}\left(1-e^{-\alpha_1 p_1(t)}\right),
\qquad t\in I_1 .
\label{eq:g1}
\end{equation}
Because $p_1$ is non-constant there exist $t_a,t_b\in I_1$ with $0<p_a:=p_1(t_a)<p_b:=p_1(t_b)$;
put $r:=p_b/p_a>1$. Taking the ratio of \eqref{eq:g1} at these instants eliminates $A_{11}$:
\begin{equation}
\frac{g_1(t_b)}{g_1(t_a)} = \frac{1-e^{-\alpha_1 p_b}}{1-e^{-\alpha_1 p_a}}
= \tilde H_r(\alpha_1 p_a).
\label{eq:ratio}
\end{equation}
The left side is observed and lies in $(1,r)$, so Lemma~\ref{lem:H} inverts \eqref{eq:ratio}
uniquely, giving $\alpha_1$ and then
$A_{11}=g_1(t_a)/\left(1-e^{-\alpha_1 p_a}\right)$.

\emph{Step 3 (cross-transfer, from S1).} On the same interval,
$\dot k_2 = A_{21}\left(1-e^{-\alpha_1 p_1(t)}\right)(1-k_2)-\Lambda_2 k_2$, whence for any
$t\in I_1$ with $p_1(t)>0$,
\begin{equation}
A_{21} = \frac{\dot k_2(t) + \Lambda_2 k_2(t)}
{\left(1-k_2(t)\right)\left(1-e^{-\alpha_1 p_1(t)}\right)} .
\label{eq:A21}
\end{equation}

\emph{Step 4 (symmetry).} Applying Steps 2--3 to $I_2$ yields $\alpha_2, A_{22}$, then $A_{12}$
(Appendix~\ref{app:thm}). All eight parameters are expressed by explicit formulas, so
$\theta=\theta'$.
\end{proof}

\subsection{Consequences}\label{sec:ident:cons}

\paragraph{Constructivity} The proof is an estimator, not an existence argument:
\eqref{eq:lam}--\eqref{eq:A21} supply a closed-form initialisation for the optimisation of
\eqref{eq:loss} and a diagnostic --- a large discrepancy between formula-based and
optimisation-based estimates signals noise or a violation of (R), (S1), (S2). A minimal
informative design for $K=2$ requires seven well-placed observation instants
(Appendix~\ref{app:thm}).

\paragraph{Robustness} In practice $\bm{k}$ is observed through \eqref{eq:emission} and the
derivatives in \eqref{eq:g1}, \eqref{eq:A21} must be estimated from smoothed trajectories.

\begin{lemma}[Robustness]\label{lem:robust}
Under the hypotheses of Theorem~\ref{thm:ident}, with $k_{\min}>0$, $u_{\min}>0$ and
$\alpha_j p$ confined to a compact subset of $(0,\infty)$, there is a constant $C^{\dagger}$
independent of the perturbation magnitudes such that
$\left\|\hat\theta-\theta^{*}\right\|_{\infty} \le C^{\dagger}\left(\epsilon+\delta\right)$,
where $\epsilon$ and $\delta$ bound the errors in $\bm{k}$ and $\dot{\bm{k}}$ respectively.
\end{lemma}

The proof and the structure of $C^{\dagger}$ are given in Appendix~\ref{app:robust}. The
dependence is as instructive as the linearity: reconstruction degrades with short rest intervals,
with mastery near either boundary, and --- most relevantly for what follows --- as the ratio of
practice intensities approaches unity, where the inversion of $\tilde H_r$ becomes arbitrarily
ill-conditioned.

\paragraph{Scope} The result is structural: it concerns noise-free observation under an admissible
schedule. It does not address the stochastic term, practical identifiability under sparse discrete
observations, or misspecification of $\bm{Q}_{\mathrm{prior}}$. The extension to arbitrary $K$ ---
with conditions (R) and (Sj) for every $j$, and induction on $j$ --- is left as a conjecture;
Section~\ref{sec:results:synth} provides numerical evidence for $K=5$.

\paragraph{The reading that matters empirically} Conditions (R), (S1), (S2) describe a schedule an
instructor could impose but that no learner follows spontaneously. Sections~\ref{sec:setup} and
\ref{sec:results} examine what remains of parameter recovery when data are collected without such
design.

%% =====================================================================
\section{Numerical methods}\label{sec:numerics}
%% =====================================================================

\subsection{Discretisation of the forward problem}

Write \eqref{eq:main} in deterministic semilinear form
\begin{equation}
\dot{\bm{k}} = \bm{a}(t)\odot(\bm{1}-\bm{k}) - \bm{\Lambda}\odot\bm{k},
\qquad \bm{a}(t) := \bm{A}\,\bm{\phi}(\bm{p}(t)) \ge \bm{0} .
\label{eq:semilinear}
\end{equation}
The system is stiff in a structural sense: acquisition acts only while practice is applied, whereas
dissipation acts always, with rates that may differ across concepts by an order of magnitude.

Because the input is piecewise constant between interactions, we integrate on the grid of
interaction instants with fixed step $h$. The experiments in Section~\ref{sec:results} use the
classical explicit fourth-order Runge--Kutta scheme with $h=0.25$ in units of the interaction grid;
the step was selected by refinement.

For deployment on long unpractised intervals, where an explicit scheme requires
$h \lesssim 2/\max_j\Lambda_j$, we also derive an implicit scheme. Since for fixed input the equation is linear in $\bm{k}$, the implicit Euler method can be applied to the whole right-hand side,
\begin{equation}
\bm{k}_{n+1} = \bigl(\bm{1}+h\,(\bm{a}(t_n)+\bm{\Lambda})\bigr)^{-1}\odot\left[\bm{k}_{n} + h\,\bm{a}(t_n)\right],
\label{eq:semiimplicit}
\end{equation}
explicit to evaluate since the matrix to invert is diagonal.

\begin{proposition}[L-stability on the dissipative subsystem]\label{prop:lstable}
Applied to $\dot k = -\Lambda k$, $\Lambda>0$, scheme \eqref{eq:semiimplicit} has amplification
factor $R(z)=(1+z)^{-1}$ with $z=h\Lambda$; hence $|R(z)|<1$ for all $h>0$ and $R(z)\to0$ as
$z\to\infty$, so the scheme is A-stable and L-stable in the sense of \cite{hairer1996solving}. It preserves the invariant box of Lemma~\ref{lem:box} for every step $h>0$.
\end{proposition}

\begin{proof}
With $\bm{a}=0$ the scheme reduces to $k_{n+1}=k_n/(1+h\Lambda)$, which gives $R(z)=(1+z)^{-1}$. For the box, the numerator $k_{n,i}+ha_i$ is non-negative, so $k_{n+1,i}\ge0$, and $k_{n,i}+ha_i\le1+ha_i<1+h(a_i+\Lambda_i)$, so $k_{n+1,i}<1$. No restriction on the step is required.
\end{proof}

L-stability matters because an explicit scheme with too large a step produces oscillatory or
negative mastery after long gaps --- outside the state space, not merely inaccurate.

\subsubsection*{Comparison of the schemes on stiff regimes}

To delimit the applicability of each scheme quantitatively, consider a model problem with two
concepts whose forgetting rates differ by a factor of forty,
$\bm{\Lambda} = (0.05,\ 2.00)\ \mathrm{day}^{-1}$. On an interval without practice the solution is
known analytically, $\bm{k}(T) = \bm{k}_0 \odot e^{-\bm{\Lambda} T}$, so the error can be measured
directly.

\begin{table}[h]
\centering
\caption{Absolute error after a gap of $T = 30$ days. ``escape'' indicates that the solution left
the invariant box $[0,1]$.}
\begin{tabular}{ccrcrc}
\toprule
$h$, days & $h\Lambda_{\max}$ & \multicolumn{2}{c}{explicit RK4} & \multicolumn{2}{c}{scheme~\eqref{eq:semiimplicit}} \\
\cmidrule(lr){3-4}\cmidrule(lr){5-6}
 & & error & box & error & box \\
\midrule
0.25 & 0.50 & $4.1\cdot10^{-11}$ & ok & $1.3\cdot10^{-3}$ & ok \\
1.00 & 2.00 & $1.1\cdot10^{-8}$  & ok & $4.9\cdot10^{-3}$ & ok \\
2.00 & 4.00 & $1.7\cdot10^{10}$  & escape & $9.8\cdot10^{-3}$ & ok \\
5.00 & 10.0 & $3.3\cdot10^{14}$  & escape & $2.3\cdot10^{-2}$ & ok \\
30.0 & 60.0 & $2.8\cdot10^{5}$   & escape & $1.1\cdot10^{-1}$ & ok \\
\bottomrule
\end{tabular}
\label{tab:schemes}
\end{table}

Table~\ref{tab:schemes} shows that the explicit scheme retains high accuracy up to its stability
limit, which for the fourth-order Runge--Kutta method on the negative real axis is
$|h\Lambda| \le 2.785$. Beyond it the solution diverges and leaves the invariant box, producing
negative mastery; with $\Lambda_{\max} = 2\ \mathrm{day}^{-1}$ this already occurs at
$h \ge 1.4$ days. By Proposition~\ref{prop:lstable}, scheme~\eqref{eq:semiimplicit} remains stable
at any step and stayed inside the box at all steps tested, with an error that grows smoothly and
does not exceed $1.1\cdot10^{-1}$ even when the step equals the whole thirty-day gap.

A practical division of roles follows. For parameter estimation from data, where the step is set by
the interaction grid and is small, the explicit scheme is preferable and is used in all experiments
of Section~\ref{sec:results}. For long horizons and for on-device fine-tuning, where the step is
dictated by computational resources rather than by accuracy, scheme~\eqref{eq:semiimplicit} is
required. We stress that its advantage is not accuracy --- at small steps the fourth-order explicit
method is more accurate by many orders of magnitude --- but unconditional stability and preservation
of the state space.

\subsection{The inverse problem: adjoint sensitivity}

Storing the full computational graph costs $O(NK)$ memory for a trajectory of $N$ steps. The
adjoint formulation solves the augmented system backwards in time; with checkpointing at
$O(\sqrt{N})$ nodes \cite{griewank2008evaluating} the cost becomes memory $O(\sqrt{N}K)$ and
compute $O(N\sqrt{N}K^{2})$, the regime required for parameter updates on a learner's own device.
The experiments here use direct backpropagation, for which the batched formulation already removes
the bottleneck.

\subsection{A batched reformulation with provable equivalence}\label{sec:numerics:batched}

The natural implementation integrates one learner at a time. This is the dominant cost: on a cohort
of $2.5\times10^{4}$ learners a single epoch took approximately $6.2$ hours on a modern GPU with
device utilisation below 20\% --- the computation is latency-bound in the per-trajectory loop.

When time is measured on the interaction grid, all learners share a common integer time axis; a
batch of $B$ learners can therefore be advanced by a single ODE solve on the union grid, with the
state held as a $(B,K)$ tensor and each learner's mastery read at that learner's own interaction
indices, padding masked out.

\begin{proposition}[Numerical equivalence]\label{prop:equiv}
For fixed parameters, the batched and per-trajectory formulations produce identical predicted
probabilities and identical gradients of \eqref{eq:loss}, up to floating-point round-off.
\end{proposition}

The proof is given in Appendix~\ref{app:batched}. We verify it empirically: on a synthetic cohort
in the exact data format of the loader, predicted probabilities agree bit-exactly
($\max|\Delta p| = 0$), gradients with respect to the transfer parameters agree to
$1.5\times10^{-10}$, and the full objective and its gradient agree to $|\Delta\mathcal{L}|=0$ and
$8.3\times10^{-10}$. The practical effect is a reduction of epoch time from $6.2$ hours to
$16$ seconds on the same cohort and hardware --- a factor of approximately $245$ --- which is what
makes the five-seed study and the four control experiments below feasible.

\subsection{Non-uniform time grids}\label{sec:numerics:time}

Section~\ref{sec:results:controls} reports a variant in which pseudo-time is replaced by real
elapsed time. This breaks the common integer axis: with continuous timestamps the union grid of a
batch is almost fully distinct, and consecutive nodes may be separated by intervals below
$10^{-5}$ while others span months. We therefore reparameterise monotonically and quantise: node
times are set to $\operatorname{cumsum}\left(\log(1+\Delta t)\right)$ with $\Delta t$ in hours,
then rounded onto a grid of resolution $0.25$ with strict monotonicity enforced. Without this
treatment the integration diverges: we observed non-finite objective values and chance-level
predictions throughout training.

%% =====================================================================
\section{Experimental setup}\label{sec:setup}
%% =====================================================================

\subsection{Data}

We evaluate on synthetic cohorts and four public benchmarks: ASSISTments-2009/2015/2017 and
Junyi-2015 \cite{chang2015junyi} ($K=39$ topic-level knowledge components). ASSISTments-2009
contains duplicated rows that inflate AUC for every model \cite{xiong2016going}; we de-duplicate
before all runs. All models share student-level folds and identical (student, step) keys;
significance is assessed with the DeLong test \cite{delong1988comparing} and a student-clustered
bootstrap, with Holm correction \cite{holm1979simple}. For Junyi we use a fixed 25\,000-student
subsample (seed 42).

\paragraph{Practice input, difficulty and rank} For every interaction, the practised concept receives
a rectangular pulse of fixed amplitude centred on the interaction time on the interaction grid.
Practice intensity therefore takes essentially a single positive value. By the necessity result in
\cite{kostanian2026exact}, the response gains $\alpha$ are then not identifiable from such data,
whatever their volume; only the products $A\,\mathrm{diag}(\phi(\bar p))$ are. Topic difficulty is
the per-topic error rate computed from the same logs. The correlation reported in
Section~\ref{sec:results:apparent} is therefore not an external validation; its strength is close to
an algebraic consequence of the stationary relation \eqref{eq:steady}. The factor $\bm{B}$ has rank $r=25$. Training uses fold~0 (17\,500 / 2\,500 / 5\,000 learners for training, validation and test) and training seeds 0--4; seed~42 refers only to the subsampling. Time is measured on the interaction grid unless stated otherwise, so that
$\Lambda$ is a rate per interaction rather than per day in these experiments; forgetting
interpreted in calendar time is examined only in the real-time variant (control C3).

\subsection{Protocol for parameter validation}\label{sec:setup:protocol}

The model is trained without an expert prior ($\bm{Q}_{\mathrm{prior}}=\bm{1}$). Unless stated
otherwise we use the batched solver, zero residual and regularisation weights, $h=0.25$, and a
fixed budget of 25 epochs without early stopping, so that no run is advantaged by a lucky stopping
point.

We do not treat a single correlation as evidence. Every parameter-recovery claim is subjected to
four controls:
\begin{itemize}
\item \textbf{C1 (shuffle).} Interaction order is permuted within each learner, destroying temporal
structure while preserving per-topic frequencies exactly.
\item \textbf{C2 (baseline).} The same external correlation is computed for a classical BKT model.
\item \textbf{C3 (real time).} Pseudo-time is replaced by real elapsed time and the recovered rates
are correlated with an empirical per-topic forgetting slope measured directly from the data.
\item \textbf{C4 (time ablation).} The same run with real time switched off.
\end{itemize}

%% =====================================================================
\section{Results}\label{sec:results}
%% =====================================================================

\subsection{Synthetic data: the theory holds under excitation}\label{sec:results:synth}

On synthetic cohorts generated under the excitation schedule of Theorem~\ref{thm:ident}, the
closed-form reconstruction recovers parameters to a relative error of $10^{-15}$ (machine
precision) for $K=2$. For $K=5$, end-to-end training recovers the prerequisite graph exactly
($F_1=1.0$) and attains AUC $=0.717$ against an oracle ceiling of $0.720$.

\subsection{Predictive accuracy: a quantified trade-off}

\begin{table}[h]
\centering
\caption{Matched AUC on contamination-controlled benchmarks. All pairwise differences are
significant (DeLong, $p\approx0$).}
\begin{tabular}{lccc}
\toprule
Dataset & Cognitive-PINN & DKT & BKT \\
\midrule
ASSISTments-2009 & 0.682 & 0.758 & 0.717 \\
ASSISTments-2015 & 0.695 & 0.730 & 0.691 \\
ASSISTments-2017 & 0.607 & 0.697 & 0.627 \\
\bottomrule
\end{tabular}
\label{tab:auc}
\end{table}

Cognitive-PINN is below DKT on all three data sets and below BKT on two of
three datasets. This is the expected price of constraining the model to a small set of
mechanistically meaningful parameters; predictive accuracy is not a contribution of this work.

A natural objection follows: what is the point of a model that is less accurate than a design three
decades old. The answer is that the models answer different questions. BKT and DKT predict the next
response, but neither admits the question of identifiability: in BKT the parameters are attached to
individual skills and describe no transfer between them, while the latent state of DKT has no
mechanistic reading with respect to which such a question would be meaningful. At the cost of
accuracy, the proposed model acquires a property absent from both: one can \emph{prove} under which
observation conditions its parameters are uniquely determined and, consequently, use it as an
instrument for testing whether those conditions hold in the data at hand. That test, rather than
predictive accuracy, is the subject of Sections~\ref{sec:results:apparent}--\ref{sec:results:why}.

\subsection{Observational data: an apparent recovery}\label{sec:results:apparent}

Trained on Junyi without an expert prior, the recovered forgetting rates $\bm{\Lambda}$ correlate
strongly with independently measured topic difficulty: Spearman $\rho = 0.83 \pm 0.01$ across five
seeds, with partial $\rho = 0.78 \pm 0.01$ after controlling for the number of observations per
topic. Read in isolation, this is the kind of evidence commonly offered as demonstrating that a
model's parameters are meaningful.

%% --- Fig.1: uncomment after running make_figures.py ---
\begin{figure}[h]
\centering
\includegraphics[width=0.72\textwidth]{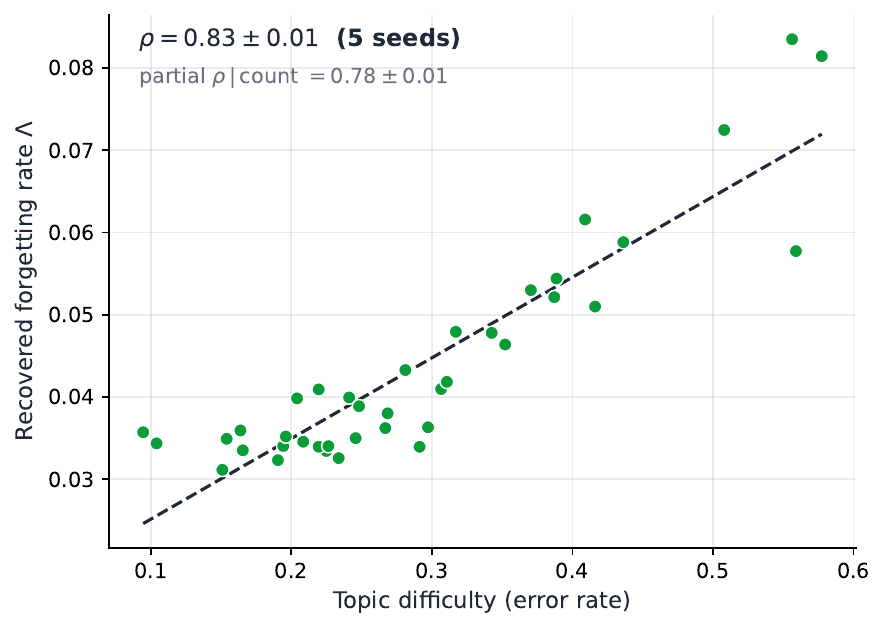}
\caption{Recovered forgetting rates versus independently measured topic difficulty on Junyi
($K=39$), across five seeds. The correlation is strong and stable --- and, as
Section~\ref{sec:results:controls} shows, not evidence of recovering forgetting.}
\label{fig:lambda}
\end{figure}

\subsection{Four controls refute the recovery interpretation}\label{sec:results:controls}

\begin{table}[h]
\centering
\caption{Controls on the apparent parameter recovery. Single values refer to one training run; the five-seed mean of the unshuffled correlation is $0.834\pm0.006$, with a 95\% percentile bootstrap interval over topics of approximately $[0.63,0.93]$. Spreads across seeds do not reflect sampling uncertainty across the 39 topics; the controls are therefore read as showing that the correlation is not destroyed, not as estimating differences between conditions.}
\begin{tabular}{llc}
\toprule
Control & Quantity & Result \\
\midrule
C1 shuffle & $\rho(\bm{\Lambda},\text{difficulty})$ after permuting order & $0.829$ (vs $0.845$) \\
C2 baseline & $\rho(\text{BKT slip},\text{difficulty})$ & $0.795$ \\
            & $\rho(\text{BKT guess},\text{difficulty})$ & $-0.699$ \\
C3 real time & $\rho(\bm{\Lambda},\text{empirical forgetting slope})$ & $0.49\pm0.01$ \\
C4 time ablation & same, real time removed & $0.504$ \\
\bottomrule
\end{tabular}
\label{tab:controls}
\end{table}

C1 shows the correlation is insensitive to temporal ordering, i.e.\ carried by per-topic answer
frequencies rather than by learning dynamics. C2 shows a classical per-skill BKT reproduces
essentially the same correlation, so it is not evidence about our model at all. C3 shows that when
the reference is switched from difficulty to an empirically measured forgetting slope, the
correlation drops to a moderate $0.49$ while the difficulty correlation remains at $0.79$. C4 is
decisive: removing real timestamps does not reduce the correlation ($0.504$ vs $0.49$), so the
timestamps contribute nothing.

Structural recovery is likewise weak and unstable: edge-AUC against the expert prerequisite graph
is $0.570 \pm 0.026$ over five seeds (range $0.534$--$0.609$), and the sparser reference graph
yields no significant signal.

%% --- Fig.2: uncomment after running make_figures.py ---
\begin{figure}[h]
\centering
\includegraphics[width=0.72\textwidth]{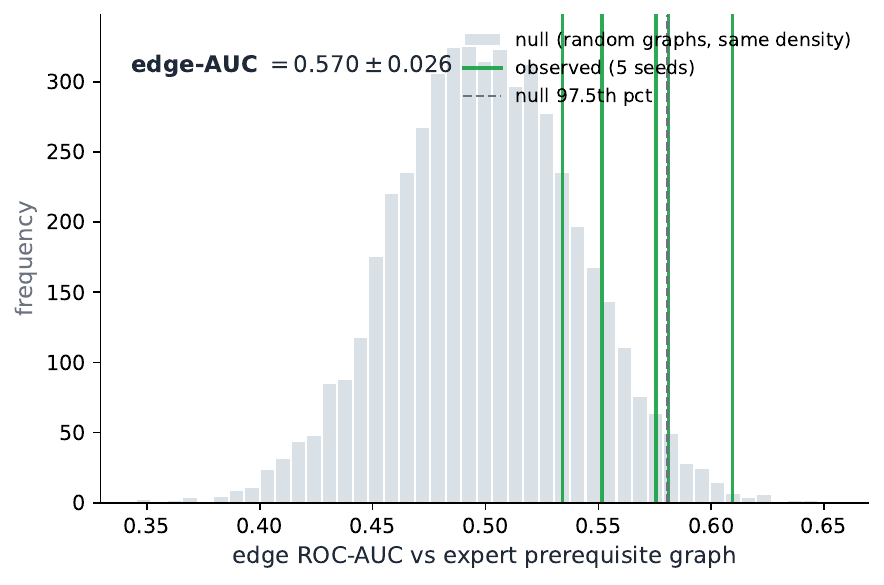}
\caption{Structural recovery across five seeds against the null distribution of random graphs of
identical density. The observed edge-AUC values lie close to the upper tail of the null, and the
spread across seeds exceeds the margin over chance.}
\label{fig:structure}
\end{figure}

\subsection{Why the recovery degenerates}\label{sec:results:why}

The behaviour follows from \eqref{eq:steady}: $\Lambda_j$ directly parameterises the predicted
response probability of a topic, so minimising the data term \emph{requires} $\Lambda_j$ to absorb
the observed per-topic error rate. The correlation with difficulty is a property of the stationary
structure, not a discovery about the learner --- which is what C2 demonstrates empirically.

A genuine forgetting signal does exist in the data: topic-controlled accuracy falls monotonically
from $0.751$ for gaps below one hour to $0.589$ for gaps beyond thirty days (within-topic Spearman
$\rho = -0.176$ across all 39 topics). But it is carried by the 2.4\% of interactions separated by
a day or more and is overwhelmed in the likelihood by the remaining 97\% of within-session
interactions. Real timestamps fail to help not because the signal is absent, but because it is not
identifiable from passive logs at the weight it carries in the objective --- a failure consistent with the analysis of Section~\ref{sec:ident} in the absence of designed excitation. Theorem~\ref{thm:ident} itself concerns noise-free continuous observation and does not by itself imply this failure.

%% =====================================================================

\subsection{Additional controls: ablation, mechanism isolation and a positive control}\label{sec:results:extra}

The analyses in this subsection use the estimator and configuration of Section~\ref{sec:results:apparent}
(batched solver, 25 epochs, learning rate $3\cdot10^{-3}$, rank 25, fold~0, training seeds 0--4);
the baseline runs reproduce the forgetting rates reported there exactly. Correlations are given as
mean $\pm$ SD over seeds; intervals are 95\% percentile bootstrap intervals over topics.

\paragraph{Ablation: a single forgetting rate} Replacing the topic-specific rates by one shared
scalar leaves accuracy unchanged (Table~\ref{tab:ablation}). Topic difficulty is carried by the
location parameter $\delta$ in both models ($\rho=0.945$, interval approximately $[0.86,0.97]$),
while the topic-specific rates, though reproducible across seeds (pairwise $\rho\ge0.995$), vary by
a factor of only $2.7$ across topics (0.031 to 0.085 per interaction step, median 0.039). The
topic-specific forgetting rates are therefore not needed for predictive accuracy, and their
correlation with difficulty accompanies rather than carries the difficulty signal. Because the
emission parameters are indexed by topic, a free per-topic intercept would coincide with $\delta$
and was not fitted separately.

\begin{table}[h]
\centering\small
\caption{Ablation on Junyi (five seeds): topic-specific versus shared forgetting rate.}
\label{tab:ablation}
\begin{tabular}{lcc}
\toprule
 & topic-specific $\Lambda$ & shared $\Lambda$ \\
\midrule
validation AUC & $0.6457\pm0.0003$ & $0.6457\pm0.0003$ \\
test AUC & $0.6553\pm0.0002$ & $0.6553\pm0.0002$ \\
$\rho(\hat\Lambda,\text{difficulty})$ & $0.834\pm0.006$ & --- ($\hat\Lambda=0.034$) \\
$\rho(\hat\delta,\text{difficulty})$ & $0.945\pm0.003$ & $0.945\pm0.004$ \\
$\rho(\hat\beta,\text{difficulty})$ & $-0.916\pm0.006$ & $-0.920\pm0.004$ \\
\bottomrule
\end{tabular}
\end{table}

\paragraph{Mechanism isolation: known rates on the real schedule} We retained the real schedule of all
25\,000 learners and replaced the responses by responses simulated from the model, with $A$,
$\alpha$ and $\beta$ taken from the fit of the same seed and $\delta$ recalibrated so that per-topic
error rates match the observed ones (absolute difference below $0.001$). In case (a) all topics share
the rate $\Lambda=0.05$; in case (b) the true rates are independent of difficulty
($\rho=0.000\pm0.006$). In both cases the fitted rates correlate with difficulty, and their range
nearly coincides with the range fitted to the real data (Table~\ref{tab:replay}). In case (b) the
fitted rates are only weakly related to the true ones: the partial correlation given difficulty is
$0.25\pm0.09$, and its interval includes zero for three of the five seeds. The correlation with
difficulty is therefore produced in large part by the schedule and the model structure, even when the
generating process contains no such relation. The simulated correlations ($0.69$--$0.72$) fall short
of the observed $0.83$; this remainder is not explained by the analysis.

\begin{table}[h]
\centering\small
\caption{Replay of the real Junyi schedule with known forgetting rates (five seeds).}
\resizebox{\textwidth}{!}{%
\label{tab:replay}
\begin{tabular}{lccc}
\toprule
 & (a) identical $\Lambda$ & (b) independent $\Lambda$ & real data \\
\midrule
$\rho(\hat\Lambda,\text{difficulty})$ & $0.72\pm0.04$ & $0.69\pm0.08$ & $0.83$ \\
$\rho(\hat\Lambda,\Lambda_{\mathrm{true}})$ & --- & $0.18\pm0.08$ & --- \\
partial $\rho(\hat\Lambda,\Lambda_{\mathrm{true}}\mid\text{difficulty})$ & --- & $0.25\pm0.09$ & --- \\
$\hat\Lambda$: min / median / max & 0.030 / 0.037 / 0.085 & 0.029 / 0.037 / 0.085 & 0.031 / 0.039 / 0.085 \\
\bottomrule
\end{tabular}}
\end{table}

\paragraph{Positive control} To test whether the controls of Section~\ref{sec:results:controls} can
recognise genuine recovery, we generated synthetic cohorts ($K=39$, $N=5000$) with designed
excitation: blocks of isolated practice at two pulse amplitudes (1 and 2.5) and rest intervals of at
least 20 grid steps with assessments at their start, middle and end; true rates were log-uniform on
$[0.01,0.2]$ per step. With the parameter values planned initially (diagonal transfer $\approx4$,
off-diagonal $0.5$, emission slopes $1$--$1.8$), the data carry almost no information about
$\Lambda$: oracle models with the true rates and with rates permuted across topics differ in AUC by
$0.0005$, and recovery is absent with and without shuffling ($\rho(\hat\Lambda,\Lambda_{\mathrm{true}})
=0.11\pm0.19$ and $0.10\pm0.20$). In this regime the controls are uninformative. We therefore selected,
by the oracle criterion alone and without fitting, values at which forgetting is observable (oracle
AUC difference above $0.01$; diagonal transfer $U[0.3,0.6]$, off-diagonal $0.05$, emission slopes
$2$--$3.5$, achieved difference $0.018$). In this regime recovery is present without shuffling and
disappears with it (Table~\ref{tab:positive}): the partial correlation of $\hat\Lambda$ with the true
rates given difficulty is $0.40\pm0.10$ against $0.03\pm0.11$, and the difference is positive in every
seed. The raw correlation is not informative here, because the simulated true rates are themselves
correlated with difficulty ($\rho=0.38$). BKT does not reproduce the true rates: its learning
probability correlates negatively with them ($-0.59\pm0.15$), and the other parameters are not
distinguishable from zero, while all four are strongly related to difficulty. The fitted rates are
strongly compressed ($0.044$--$0.073$ for true values $0.01$--$0.2$). The controls are thus specific
where forgetting is observable and uninformative where it is not; the regime in which they are
specific was selected for sensitivity and uses emission slopes steeper than those fitted to Junyi
($\approx1.4$).

\begin{table}[h]
\centering\small
\caption{Positive control under designed excitation with observable forgetting (five seeds).}
\label{tab:positive}
\begin{tabular}{lcc}
\toprule
 & original order & shuffled (C1) \\
\midrule
validation AUC & $0.654\pm0.011$ & $0.644\pm0.011$ \\
partial $\rho(\hat\Lambda,\Lambda_{\mathrm{true}}\mid\text{difficulty})$ & $0.40\pm0.10$ & $0.03\pm0.11$ \\
$\rho(\hat\Lambda,\text{difficulty})$ & $0.98$ & $0.99$ \\
\midrule
BKT (C2): $\rho(\cdot,\Lambda_{\mathrm{true}})$ / $\rho(\cdot,\text{difficulty})$ & & \\
\quad prior & \multicolumn{2}{c}{$-0.27\pm0.28$ / $-0.94$} \\
\quad learn & \multicolumn{2}{c}{$-0.59\pm0.15$ / $-0.85$} \\
\quad slip & \multicolumn{2}{c}{$0.29\pm0.06$ / $0.35$ (undefined in 2 of 5 seeds)} \\
\quad guess & \multicolumn{2}{c}{$0.14\pm0.18$ / $-0.50$} \\
\bottomrule
\end{tabular}
\end{table}

\section{Discussion}\label{sec:discussion}
%% =====================================================================

\paragraph{The failure is consistent with the robustness analysis} Lemma~\ref{lem:robust} bounds the
reconstruction error by $C^{\dagger}(\epsilon+\delta)$, but its constant degrades with short rest
intervals, with mastery near the boundaries of its range, and --- most relevantly --- as the ratio
of practice intensities approaches unity, where the inversion underlying recovery of the practice
gain becomes arbitrarily ill-conditioned. Observational logs exhibit exactly this combination: gaps
are overwhelmingly short (96.9\% below one hour) and practice intensity is close to constant. The
degeneration of Section~\ref{sec:results:controls} is therefore consistent with the regime that the robustness analysis identifies as ill-conditioned, although the analysis does not translate sparse binary responses into an effective noise level.

\paragraph{A sharp boundary, predicted by theory and observed in data}
Sections~\ref{sec:results:synth} and \ref{sec:results:controls} describe the same model on two
kinds of data. Where the excitation conditions hold, recovery is exact. Where they do not --- in
observational logs, however large --- recovery degenerates into a reparameterisation of answer
frequencies. The identifiability theorem is therefore not a technical footnote: it delineates a
qualitative boundary between two regimes of the same dynamical system, and the boundary is
empirically visible. For complex social systems modelled from passive observation, the practically
relevant statement is that scale of data does not substitute for design of observation.

\paragraph{A methodological consequence} Interpretable knowledge tracing routinely supports its
claims by correlating a learned quantity with an external property. Controls C1 and C2 show that
such a correlation can be strong ($\rho=0.83$), survive the destruction of all temporal
information, and be matched by a three-decade-old baseline. It is therefore not sufficient evidence
that a model has recovered the mechanism it names. We recommend that claims of parameter
interpretability be accompanied by, at minimum, a shuffle control and a simple-baseline control;
both are inexpensive and, in our case, decisive. This critique applies to our own initial reading
of the data, which the controls overturned.

\paragraph{On negative results and identifiability} The contribution is not that a particular model
failed to fit. It is that a model whose parameters are provably identifiable under stated
conditions can be used as an instrument to test whether those conditions are met in a given
dataset --- and that, for the standard benchmarks of this field, they are not. This reframes a
recurring difficulty in educational data mining --- unstable, implausible parameter estimates ---
as a question of experimental design rather than of model capacity, and connects it to the
established distinction between structural and practical identifiability
\cite{wieland2021identifiability}.

\paragraph{Relation to subsequent results} A companion analysis \citep{kostanian2026exact}
places the observations above on a formal footing. The model class admits an exact discretisation,
which yields identifiability from discrete observations for an arbitrary number of concepts; a
single practice intensity renders the response gains unidentifiable; and, under a logistic
observation model, observations in a stationary regime confound the forgetting rate with the item
location parameter. The last result also answers why the difficulty signal can be absorbed by
$\Lambda$ rather than by the item difficulties of the observation model: in the stationary regime the
two are exchangeable, and the division between them is determined by initialisation and regularisation rather than by the likelihood. The ablation of Section~\ref{sec:results:extra} is consistent with this: a single shared rate attains the same accuracy, while $\delta$ carries topic difficulty.

\paragraph{Limitations} The positive control is specific only under parameter values at which forgetting is observable, and these values were selected for sensitivity; with the values planned initially the controls are uninformative. Emission parameters are indexed by topic, all results use a single fold, and the replay of Section~\ref{sec:results:extra} inherits the fitted values of $A$, $\alpha$, $\beta$ and $\delta$. The replay reproduces $0.69$--$0.72$ of the observed correlation of $0.83$; the remainder is unexplained. Theorem~\ref{thm:ident} is proved constructively for $K=2$; the $K=5$
synthetic study is empirical evidence of generalisation, and the general-$K$ case remains open.
External structural validation rests on a single dataset with an expert graph. The stochastic term
is not estimated. The real-time variant relies on a log-compressed, quantised grid, though C4
indicates the limitation is one of information content rather than of discretisation.

\paragraph{Future work} The natural continuation follows from the theorem: designed excitation. A
modest intervention --- scheduled rest intervals and blocks of isolated, varied-intensity practice
in a live course --- would supply the regimes the theory requires, turning an unidentifiable
observational problem into an identifiable experimental one. Second, the memory-efficient adjoint
scheme enables on-device parameter estimation without transmitting raw learner logs.

%% =====================================================================
\section{Conclusion}\label{sec:conclusion}
%% =====================================================================

We have presented Cognitive-PINN, a dissipative ODE model of knowledge dynamics whose parameters
carry explicit mechanistic meaning, together with a structural identifiability theorem with
constructive recovery, supporting stability and robustness results, an L-stable implicit
scheme, and a batched solver numerically equivalent to the per-trajectory formulation while being
two orders of magnitude faster.

Our empirical study is deliberately two-sided. Under the excitation conditions required by the
theorem, the framework recovers parameters and coupling structure exactly. On large observational
benchmarks it does not: an apparently strong parameter recovery ($\rho=0.83$) is shown by four controls, an ablation and a replay of the real schedule with known rates to be generated in large part by the schedule and the model structure rather than by topic-specific forgetting: it is reproduced by a classical baseline, unaffected by the removal of temporal information, not needed for accuracy, and largely recovered when the true rates are identical across topics. We trace this to the
stationary structure of the model and to the distribution of practice gaps in the data, and show that it is consistent with the failure mode anticipated by the theory.

The resulting picture is a sharp, testable boundary between identifiable and unidentifiable regimes
of the same system, and a concrete methodological recommendation for a field that increasingly
reports interpretable parameters without testing whether they are identifiable at all.

%% =====================================================================
\section*{CRediT authorship contribution statement}
%% =====================================================================
\textbf{Arman A. Kostanian:} Conceptualization, Methodology, Software, Formal analysis,
Investigation, Writing --- original draft, Visualization. \textbf{Armen L. Beklaryan:}
Supervision, Methodology, Writing --- review \& editing.

\section*{Declaration of generative AI use}
During the preparation of this work the authors used Claude (Anthropic) in order to draft and edit
text and to write analysis code. After using this tool, the authors reviewed and edited the content as
needed and take full responsibility for the content of the article.

\section*{Declaration of competing interest}
The authors declare that they have no known competing financial interests or personal relationships
that could have appeared to influence the work reported in this paper.

\section*{Data availability}
Code is publicly available at \url{https://github.com/armankostanian/cognitive-pinn}. The datasets
are public knowledge-tracing benchmarks; raw learner logs are not redistributed.

%% =====================================================================
\appendix
\section{Completion of the proof of Theorem~\ref{thm:ident}}\label{app:thm}
%% =====================================================================

\subsection{Step 4}
On $I_2$ we have $p_1\equiv0$, hence $\phi_1=0$, and the second component of \eqref{eq:K2} reduces
to $\dot k_2 = A_{22}\left(1-e^{-\alpha_2 p_2(t)}\right)(1-k_2)-\Lambda_2 k_2$. With $\Lambda_2$
known, define $g_2(t) := \left(\dot k_2 + \Lambda_2 k_2\right)/\left(1-k_2\right)
= A_{22}\left(1-e^{-\alpha_2 p_2(t)}\right)$. Since $p_2$ is non-constant, choose
$t_c,t_d\in I_2$ with $0<p_c:=p_2(t_c)<p_d:=p_2(t_d)$ and $s:=p_d/p_c>1$; then
$g_2(t_d)/g_2(t_c)=\tilde H_s(\alpha_2 p_c)$, and Lemma~\ref{lem:H} gives $\alpha_2$ and
$A_{22}=g_2(t_c)/\left(1-e^{-\alpha_2 p_c}\right)$. Finally, from the first component on the same
interval,
\begin{equation}
A_{12} = \frac{\dot k_1(t) + \Lambda_1 k_1(t)}
{\left(1-k_1(t)\right)\left(1-e^{-\alpha_2 p_2(t)}\right)} .
\end{equation}

\subsection{Range of the observed ratio}
Step 2 requires $\rho:=g_1(t_b)/g_1(t_a)\in(1,r)$. That $\rho>1$ follows from strict monotonicity
of $p\mapsto1-e^{-\alpha p}$; that $\rho<r$ follows from Lemma~\ref{lem:H} with $\alpha_1>0$, the
value $r$ being attained only as $\alpha_1 p_a\to0^{+}$. The excluded limits are the two degenerate
regimes of the response function --- vanishing gain and complete saturation --- both of which
reappear quantitatively in Appendix~\ref{app:robust} as sources of ill-conditioning.

\subsection{Minimal informative design}
The proof uses two instants on $I_R$ per concept, two instants at distinct intensities on each of
$I_1,I_2$, and one further instant on each for the off-diagonal entries: seven informative instants
for $K=2$ in the noise-free setting.

%% =====================================================================
\section{Proof of Lemma~\ref{lem:robust} and its empirical range of validity}\label{app:robust}
%% =====================================================================

\subsection{Setting}
Let $\hat{\bm{k}}=\bm{k}+\bm{\xi}$ with $\|\bm{\xi}\|_{\infty}\le\epsilon$ and
$\|\hat{\dot{\bm{k}}}-\dot{\bm{k}}\|_{\infty}\le\delta$. Put $k_{\min}:=\inf k_i>0$,
$u_{\min}:=\inf(1-k_i)>0$ (both positive by Lemma~\ref{lem:box}) and
$\Lambda_{\max}:=\max_i\Lambda_i$.

\subsection{Dissipation rates}
From \eqref{eq:lam}, with $\Delta t := t-t_R^{-}$,
$\hat\Lambda_i-\Lambda_i = -\Delta t^{-1}\left[\log(\hat b/b)-\log(\hat a/a)\right]$. Using
$|\log(1+x)|\le2|x|$ for $|x|\le1/2$,
\begin{equation}
\left|\hat\Lambda_i-\Lambda_i\right| \le \frac{4\epsilon}{\Delta t\,k_{\min}} .
\label{eq:B1}
\end{equation}

\subsection{The observable $g_1$}
Write $g_1=n/u$ with $n:=\dot k_1+\Lambda_1 k_1$, $u:=1-k_1\ge u_{\min}$. Then
$|\hat n-n|\le\delta+\Lambda_{\max}\epsilon+|\hat\Lambda_1-\Lambda_1|$ and
$\left|\hat u^{-1}-u^{-1}\right|\le\epsilon/u_{\min}^{2}$ for $\epsilon<u_{\min}/2$, so with
\eqref{eq:B1},
\begin{equation}
\left|\hat g_1(t)-g_1(t)\right| \le C_g\left(\epsilon+\delta\right).
\label{eq:B2}
\end{equation}

\subsection{The gain}
With $\rho:=g_1(t_b)/g_1(t_a)$ and $g_{\min}:=\inf_{I_1}g_1>0$, the ratio bound and \eqref{eq:B2}
give $\left|\hat\rho-\rho\right| \le 2C_g\rho_{\max}g_{\min}^{-1}(\epsilon+\delta)$. By
Lemma~\ref{lem:H} and the inverse function theorem, for $u$ in a compact
$[u_*,u^{*}]\subset(0,\infty)$ we have
$m_r := \min_{[u_*,u^*]}\left|\tilde H_r'\right|>0$ and
$\left|\hat u-u\right|\le m_r^{-1}\left|\hat\rho-\rho\right|$, whence
\begin{equation}
\left|\hat\alpha_1-\alpha_1\right| \le \frac{C_\alpha}{m_r\,g_{\min}\,p_a}\left(\epsilon+\delta\right).
\label{eq:B4}
\end{equation}
Differentiating $A_{11}=g_1(t_a)/(1-e^{-\alpha_1 p_a})$ and inserting \eqref{eq:B2},
\eqref{eq:B4} yields a linear bound for $A_{11}$; the off-diagonal entry \eqref{eq:A21} involves no
inversion and obeys a linear bound directly. Collecting the four estimates gives
Lemma~\ref{lem:robust} with
$C^{\dagger}=C^{\dagger}\!\left(k_{\min},u_{\min},\Lambda_{\max},m_r,g_{\min},p_a,\tau_R,\tau_1,\tau_2\right)$.

\subsection{What the constant contains}
Reconstruction degrades when the rest interval is short ($\tau_R\to0$, by \eqref{eq:B1}), when
mastery is near either boundary ($k_{\min}\to0$ or $u_{\min}\to0$), and when the two practice
intensities are close ($r\to1$, hence $m_r\to0$). The third item is the quantitative counterpart of
condition (S1): non-constant practice is not merely a requirement for uniqueness --- the
\emph{degree} of variation controls the conditioning of the entire reconstruction.

\subsection{Empirical range of validity}
We reconstruct known parameters from trajectories perturbed at several noise levels
(Table~\ref{tab:noise}).

\begin{table}[h]
\centering
\caption{Relative reconstruction error versus observation-noise amplitude ($K=2$).}
\begin{tabular}{lcccccc}
\toprule
$\epsilon$ & $0$ & $0.005$ & $0.010$ & $0.020$ & $0.050$ & $0.100$ \\
\midrule
rel.\ error & $0.0$ & $0.107$ & $0.180$ & $0.75$ & $34.2$ & $30.4$ \\
\bottomrule
\end{tabular}
\label{tab:noise}
\end{table}

For $\epsilon\le0.01$ the growth is compatible with the linear bound, and the noise-free
reconstruction is exact to machine precision, confirming Theorem~\ref{thm:ident} numerically.
Beyond $\epsilon\approx0.02$ the error grows abruptly. This is not a failure of the bound but a
departure from its hypotheses: once the perturbation is large enough that $\hat\rho$ leaves the
range $(1,r)$ of $\tilde H_r$, the inversion in Step 2 has no solution. The practical reading is
that closed-form reconstruction requires low observation noise or averaging over many informative
instants, and that $r$ should be chosen comfortably above unity.

\subsection{Design corollary}
Lemma~\ref{lem:robust} converts Theorem~\ref{thm:ident} into instructions for data collection:
schedule genuine rest intervals of non-negligible length, avoid evaluating learners only near
mastery floors or ceilings, and vary practice intensity by a factor of at least two within
isolated-practice blocks. It is precisely the absence of these features from passively collected
logs that Section~\ref{sec:results} documents empirically.

%% =====================================================================
\section{Proof of Proposition~\ref{prop:equiv}}\label{app:batched}
%% =====================================================================

\emph{(i) Identity of the vector field.} For a single learner the acquisition term of
\eqref{eq:semilinear} is $\bm{A}\bm{\phi}$ with $\bm{\phi}\in\mathbb{R}^{K}$. For a batch, writing
$\bm{\Phi}\in\mathbb{R}^{B\times K}$ for the stacked responses, the corresponding expression is
$\bm{\Phi}\bm{A}^{\top}$, whose $(b,i)$ entry is
$\sum_j \Phi_{bj}A_{ij} = \left[\bm{A}\bm{\phi}^{(b)}\right]_i$. The two agree entry by entry; the
remaining terms act element-wise and are unaffected by stacking.

\emph{(ii) Locality of the integrator.} A Runge--Kutta scheme with fixed step computes
$\bm{k}_{n+1}$ from $\bm{k}_n$ and evaluations on $[t_n,t_{n+1}]$ only. Hence integrating learner
$b$ on the longer padded grid produces the same values at nodes $n\le N_b$ as integrating on
$\{0,\dots,N_b\}$: continuation beyond $N_b$ cannot influence earlier nodes.

\emph{(iii) Identity of the objective.} Because interaction times are integers, observations fall
exactly on grid nodes, so the gather used in the batched path returns the same states that
interpolation returns in the per-trajectory path. The data term is a masked sum over interactions;
summation order is the only difference. Both paths therefore compute the same function of the
parameters, and their gradients coincide by uniqueness of the derivative. \qed

\paragraph{Remark on non-integer grids} The proposition relies on the common integer axis. With
continuous node times the union grid has size $O(\sum_b N_b)$ rather than $O(T_{\max})$, and the
equivalence, while formally preserved, ceases to be useful: the computation becomes
memory-infeasible and, with node spacings spanning six orders of magnitude, numerically unstable.
The quantised reparameterisation of Section~\ref{sec:numerics:time} restores the hypotheses.

%% =====================================================================
\bibliographystyle{elsarticle-num}
\bibliography{references}
%% For local test compile without elsarticle, use:
% \bibliographystyle{unsrt}

\end{document}